\documentclass[11pt,twoside]{article}
\usepackage[margin=1in]{geometry}
\usepackage{microtype}
\usepackage{graphicx}
\usepackage{subfigure}
\usepackage{booktabs}
\usepackage{xcolor}
\usepackage{algorithm}
\usepackage{algorithmic}
\usepackage{hyperref}
\usepackage[normalem]{ulem}
\usepackage{authblk}

\usepackage[utf8]{inputenc} 
\usepackage[T1]{fontenc}    

\usepackage{natbib}

\usepackage{tikz,pgfplots}
\usetikzlibrary{fadings}

\usepackage[font=footnotesize,format=hang]{caption}
\usepackage{graphicx}
\usepackage[export]{adjustbox}
\usepackage[inline]{enumitem}

\usepackage{tikz}
\usetikzlibrary{positioning}

\usepackage{amsmath,amssymb,amsthm}
\usepackage[mathscr]{eucal}
\usepackage{stmaryrd}
\usepackage{accents}
\usepackage{upgreek}

\usepackage[most]{tcolorbox}
\usepackage{setspace}

\usepackage{booktabs,colortbl,multirow}

\graphicspath{{figs//}}

\allowdisplaybreaks

\newcommand{\indicator}[1]{\llbracket #1 \rrbracket}

\newcommand{\argmin}{{\operatorname{argmin }}}

\newcommand{\CCal}{\mathscr{C}}

\newcommand{\FCal}{\mathscr{F}}

\newcommand{\RCal}{\mathscr{R}}
\newcommand{\SCal}{\mathscr{S}}

\newcommand{\XCal}{\mathscr{X}}

\newcommand{\rv}{\pmb{r}}

\newcommand{\uv}{\pmb{u}}

\newcommand{\wv}{\pmb{w}}
\newcommand{\xv}{\pmb{x}}

\newcommand{\btheta}{\pmb{\theta}}

\newcommand{\pxy}{\mathrm{P_{XY}}}

\newcommand{\pxyi}[1]{\mathrm{P}^{#1}_{\mathrm{XY}}}

\newcommand{\Real}{\mathbb{R}}

\newcommand{\rsp}{\mathrm{reg}_{\mathrm{sp}}}
\newcommand{\rsptop}{\mathrm{reg}^{\mathrm{top}}_{\mathrm{sp}}}
\newcommand{\ellcon}{\ell_{\mathrm{cl}}}
\newcommand{\ellccon}{\ell_{\mathrm{ccl}}}

\newcommand{\ellps}{\ell_{\mathrm{sp}}}
\newcommand{\ellsp}{\ell_{\mathrm{sp}}}

\newcommand{\dist}[2]{\pmb{d}(#1, #2)}
\newcommand{\distsq}[2]{\pmb{d}^2(#1, #2)}
\newcommand{\cdist}[2]{\pmb{d}_{\mathrm{cos}}(#1, #2)}

\newcommand{\Top}[1]{{\mathrm{Top}_{#1}}}
\newcommand{\Near}[2]{{\mathcal{N}_{#1}(#2)}}

\newtheorem{lemma}{Lemma}
\newtheorem{claim}{Claim}
\newtheorem{definition}{Definition}
\newtheorem{theorem}{Theorem}

\newtheorem{proposition}{Proposition}

\theoremstyle{definition}
\newtheorem{remark}{Remark}

\newtcbtheorem[number within=section]%
{remark_tcb} 
{Remark} 
{%
fonttitle=\bfseries\upshape,fontupper=\normalsize,
colframe=green!10!black,colback=green!2!white,
colbacktitle=green!20!white,coltitle=blue!75!black,
boxsep=1pt,left=2pt,right=2pt,top=1pt,bottom=1pt,
frame hidden,boxrule=1pt,
theorem style=plain} 
{rem} 

\newtcbtheorem[number within=section]%
{assumption_tcb} 
{Assumption} 
{%
fonttitle=\bfseries\upshape,fontupper=\normalsize,
colframe=blue!5!black,colback=blue!2!white,
colbacktitle=blue!10!white,coltitle=blue!75!black,
boxsep=1pt,left=2pt,right=2pt,top=1pt,bottom=1pt,
frame hidden,boxrule=1pt,
theorem style=plain} 
{asm} 

\newtcbtheorem[number within=section]%
{example_tcb} 
{Example} 
{%
fonttitle=\bfseries\upshape,fontupper=\normalsize,
colframe=green!10!black,colback=green!2!white,
colbacktitle=green!20!white,coltitle=blue!75!black,
boxsep=1pt,left=2pt,right=2pt,top=1pt,bottom=1pt,
frame hidden,boxrule=1pt,
theorem style=plain} 
{ex} 

\newcommand{\cifars}{{\sc CIFAR-10}}
\newcommand{\cifarl}{{\sc CIFAR-100}}
\newcommand{\resnet}{{\sc ResNet}}
\newcommand{\resnets}{{\sc ResNets}}

\newcommand{\lshtc}{{\sc WikiLSHTC}}
\newcommand{\amsmall}{{\sc Amazon670K}}

\newcommand{\amcat}{{\sc AmazonCat}}

\usepackage{xcolor}

\allowdisplaybreaks

\title{Federated Learning with Only Positive Labels}
\author{Felix X. Yu, Ankit Singh Rawat, Aditya Krishna Menon, and Sanjiv Kumar}
\affil{\normalsize Google Research\\
  New York, NY 10011\\
  \texttt{\{felixyu, ankitsrawat, adityakmenon, sanjivk\}@google.com}.}

\begin{document}
\maketitle

\begin{abstract}
We consider learning a multi-class classification model in the federated setting, where each user has access to the positive data associated with only a single class. As a result, during each federated learning round, the users need to locally update the classifier without having access to the features and the model parameters for the negative classes. Thus, naively employing conventional decentralized learning such as the distributed SGD or Federated Averaging may lead to trivial or extremely poor classifiers. In particular, for the embedding based classifiers, all the class embeddings might collapse to a single point.

To address this problem, we propose a generic framework for training with only positive labels, namely {\em Federated Averaging with Spreadout} ({FedAwS}), where the server imposes a geometric regularizer after each round to encourage classes to be spreadout in the embedding space. We show, both theoretically and empirically, that FedAwS can almost match the performance of conventional learning where users have access to negative labels. We further extend the proposed method to the settings with large output spaces.
\end{abstract}

\section{Introduction}
\label{sec:intro}

We consider learning a classification model in the federated learning \citep{mcmahan2017communication} setup, where each user has only access to a single class. The users are not allowed to communicate with each other, nor do they have access to the classification model parameters associated with other users' classes. Examples of such settings include decentralized training of face recognition models or speaker identification models, where in addition to the user specific facial images and voice samples, the classifiers of the users also constitute sensitive information that cannot be shared with other users.

In this work, we assume that the classification models are ``embedding-based'' discriminative models: both the classes and the input instance are embedded into the same space, and the similarity between the class embedding and the input embedding (\emph{a.k.a.}~logit or score) captures the likelihood of the input belonging to the class. A popular example of this framework are neural network based classifiers. Here, given an input instance $\xv \in \XCal$, a neural network $g_{\btheta}: \XCal \to \Real^d$ (parameterized by $\btheta$) embeds the instance into a $d$ dimensional vector $g_{\btheta}(\xv)$. The class embeddings are learned as a matrix $W \in \mathbb{R}^{C \times d}$, commonly referred to as the classification matrix, where $C$ denotes the number of classes. Finally, the logits for the instance $\xv$ are computed as $W\cdot g_{\btheta}(\xv)$.

In the federated learning setup, one collaboratively learns the classification model with the help of a server which facilitates the iterative training process by keeping track of a global model. During each round of the training process,
\begin{itemize}[noitemsep,topsep=0pt,leftmargin=0.2in]
    \item The server sends the current global model to a set of participating users.
    \item Each user updates the model with its local data, and sends the model delta to the server.
    \item The server averages (``Federated Averaging'') the deltas collected from the participating users and updates the global model.
\end{itemize}
Notice that the conventional synchronized distributed SGD falls into the federated learning framework if each user runs a single step of SGD, and the data at different users is i.i.d. Federated learning has been widely studied in distributed training of neural networks due to its appealing characteristics such as leveraging the computational power of edge devices~\cite{li2019federated}, removing the necessity of sending user data to server~\cite{mcmahan2017communication}, and various improvements on trust/security \citep{bonawitz2016practical}, privacy \citep{agarwal2018cpsgd}, and fairness \citep{mohri2019agnostic}. 

However, conventional federated learning algorithms are not directly applicable to the problem of learning with only positive labels due to two key reasons: First, the server cannot communicate the full model to each user. Besides sending the instance embedding model $g_{\btheta}(\cdot)$, for the $i$-th user, the server can communicate only the class embedding vector $\wv_i$ associated with the positive class of the user. Note that, in various applications, the class embeddings constitute highly sensitive information as they can be potentially utilized to identify the users.  

Second, when the $i$-th user updates the model using its local data, it only has access to a set of instances $\xv \in \mathcal{X}_i$ from the $i$-th class along with the class embedding vector $\wv_i$. While training a standard embedding-based multi-class classification models, the underlying loss function encourages two properties: i) similarity between an instance embedding and the positive class embedding should be as large as possible; and ii) similarity between the instance embedding and the negative class embeddings should be as small as possible. In our problem setting, the latter is not possible because the user does not have access to the negative class embeddings. 

In other words, if we were to use the vanilla federated learning approach, we would essentially be minimizing a loss function that only encourages small distances between the instances and their positive classes in the embedding space. As a result, this approach would lead to a trivial optimal solution where all instances and classes collapse to a single point in the embedding space.

To address this problem, we propose \emph{Federated Averaging with Spreadout} (FedAwS) framework, where in addition to Federated Averaging, the server applies a geometric regularization to make sure that the class embeddings are well separated (cf.~Section \ref{sec:proposed}). This prevents the model from collapsing to the aforementioned trivial solution. To the best of our knowledge, this is the first principled approach for learning in the federated setting without explicit access to negative classes.
We further show that the underlying regularizer can be suitably modified to extend the FedAwS framework to settings with large number of classes. This extension is crucial for the real-world applications such as user identification models with a large number of users. Subsequently, we theoretically justify the FedAwS framework by showing that it approximates the conventional training settings with a loss function that has access to both positive and negative labels (cf.~Section \ref{sec:analysis}). We further confirm the effectiveness of the proposed framework on various standard datasets in Section~\ref{sec:exp}. 
Before presenting our aforementioned contributions, we begin by discussing the related work  and formally describing the problem setup in Section~\ref{sec:related} and \ref{sec:background}, respectively. 
\section{Related Works}
\label{sec:related}

To the best of our knowledge, this is the first work addressing the novel setting of distributed learning with only positive labels in the federated learning framework. The learning setting we are considering is related the positive-unlabeled (PU) setting where one only has access to the positives and unlabeled data. Different from PU learning \citep{liu2002partially, elkan08, plessis15,hsiehb15}, in the federated learning setting, the clients do not have access to unlabeled data for both positive and negative classes. The setting is also related to one-class classification \cite{moya1996network, manevitz2001one} used in applications such as outlier detection and novelty detection. Different from one-class classification, we are interested in collaboratively learning a multi-class classification model.

We consider the setting of learning a discriminative embedding-based classifier. Popular neural networks fall in this category. 
An alternative approach is to train generative models. For example, each user can learn a generative model based on its own data, and the server performs the MAP estimation during the inference time. We do not consider this approach because it does not fit into the federated learning framework, where the clients and server collaboratively train a model. In addition, training a good generative model is both data and computation consuming. 
Another possible generative approach is to use federated learning to train a GAN model to synthesize negative labels for each user possibly using the techniques proposed in \citep{augenstein2019generative} and therefore convert the problem into learning with both positives and negatives. Training a GAN model in the federated setting is a separate and expensive process. In this paper we consider the setting where the users do not have access to either true or synthesized negatives.

As mentioned in the introduction, a typical application of federated learning with only positive labels is to use this learning framework to train user identification models such as speaker/face recognition models. Although the proposed FedAwS algorithm promotes user privacy by not sharing the data among the users or with the server, FedAwS itself does not provide formal privacy guarantees. To show formal privacy guarantees, we notice that differential privacy methods for federated learning \citep{agarwal2018cpsgd, abadi2016deep} can be readily employed in FedAwS by adding noise to the updates sent from each user.

On the technical side, the proposed FedAwS can be seen as using stochastic negative mining to improve spreadout regularizer.
The stochastic negative mining method was first proposed in \citep{reddi2018stochastic} to mine hard negative classes for each data point. Differently, we mine hard negative classes for each class. 
The spreadout regularization was first proposed to improve learning discriminative visual descriptors \citep{zhang2017learning} and further used in the extreme-multiclass classification setting \citep{guo2019breaking}. The spreadout regularization is related to the design of error-correcting output code (ECOC) matrix \citep{dietterich1991error, pujol2006discriminant}. In order for the ECOC matrix to work, the class embeddings have to be well separated from each other. 
In particular, similar to Proposition \ref{thm:error}, \citet{yu2013designing} shows that the classification error can be bounded by the distance between data and positive label in the embedding space, and a measure of spreadout of the classes. Differently, our result is on the true error instead of the empirical error.


\section{Problem Setup}
\label{sec:background}

\subsection{Federated learning of a classification model}
Let us first consider the conventional federated learning of a classification model, when each client has access to data from multiple classes. Let the instance space be $\XCal$, and suppose there are $C$ classes indexed by the set $[C]$. Let $\FCal \subseteq \{f : \XCal \to \Real^{C}\}$ be a set of scorer functions, where each scorer, given an instance $\xv$, assigns a score to each of the $C$ classes. In particular, for $c \in [C]$, $f(\xv)_c$ represents the relevance of the $c$-th class for the instance $\xv$, as measured by the scorer $f \in \FCal$. We consider scorers of the form
\begin{align}
\label{eq:id-model-def}
f(\xv) = Wg_{\btheta}(\xv),
\end{align}
where $g_{\btheta} : \XCal \to \Real^d$ maps the instance $\xv$ to a $d$-dimensional embedding, and $W \in \Real^{C \times d}$ uses this embedding to produce the scores ({\em a.k.a} logits) for $C$ classes as $Wg_{\btheta}(\xv)$. The $c$-th row of $W$, $\wv_c$, is referred to as the \emph{embedding vector} of the $c$-th class. The score of the $c$-th class is thus $\wv_c^T g_{\btheta}(\xv)$.

Let us assume a distributed setup with $m$ clients. In the traditional federated learning setup, for $i \in [m]$, the $i$-th client has access to $n_i$ instance and label pairs $\SCal^{i} = \{(\xv^{i}_{1}, y^{i}_{1}),\ldots, (\xv^{i}_{n_i}, y^{i}_{n_i})\} \subset \XCal \times [C]$ distributed according to an unknown distribution $\pxyi{i}$, i.e., $(\xv^{i}_j, y^{i}_j) \sim \pxyi{i}$. Let $\SCal = \cup_{i \in [m]}\SCal^{i}$ denote the set of $n = \sum_{i \in [m]}n_i$ instance and label pairs collectively available at all the clients. Our objective is to find a scorer in $\FCal$ that captures the true relevance of a class for a given instance. 

Formally, let $\ell : \Real^C \times [C] \to \Real$ be a loss function such that $\ell(f(\xv), y)$ measures the quality of the scorer $f$ on $(\xv, y)$ pair. The client minimizes an empirical estimate of the risk based on its local observations $\SCal^{i}$ as follows:
\begin{align}
\hat{f} = \argmin_{f \in \FCal} \hat{\RCal}(f; S^{i}) := \frac{1}{n_i}\sum_{j \in [n_i]}\ell\big(f(\xv^i_j), y^i_j\big).
\label{eq:loss}
\end{align}

In the federated learning setting, the $m$ clients are interested in collaboratively training a single classification model on their joint data. A coordinator server facilitates the joint iterative distributed training as follows: 
\begin{itemize}[noitemsep,topsep=0pt,leftmargin=0.2in]
\item At the $t$-th round of training, the coordinator sends the current model parameters $\btheta_{t}$ and $W_t$ to all clients.
\item For $i \in [m]$, the $i$-th client updates the current model based on its \emph{local} empirical estimate of the risk\footnote{In the federated learning setup, the client may also update the model with a few steps, not just a single step.}:
\begin{align}
&{\btheta}^{i}_t = \btheta_t - \eta\cdot \nabla_{\btheta_t} \hat{\RCal}(f_t; \SCal^{i}).\\
&{W}^i_t = W_t - \eta\cdot \nabla_{W_t} \hat{\RCal}(f_t; \SCal^{i}).
\end{align}
\item The coordinator receives the updated model parameters from all clients $\{\btheta^{i}_t, W^i_t\}_{i \in [m]}$, and updates its estimate of the model parameters using \emph{Federated Averaging}:
\begin{align}
\btheta_{t+1} = \sum_{i \in [m]}\omega_i \cdot \btheta^{i}_t;\quad W_{t+1} = \sum_{i \in [m]}\omega_i \cdot W^{i}_t, 
\end{align}
where $\pmb{\omega} = (\omega_1,\ldots, \omega_m)$ denotes the weights that the coordinator assigns to the training samples of different clients. For example, $\omega_i = \frac{n_i}{n}$ assigns uniform importance to all the training samples across different clients\footnote{Recently, \citet{mohri2019agnostic} proposed the {\em agnostic federated learning} framework to account for the heterogeneous data distribution across the clients, which crucially rely on the selecting the non-uniform weights. In this paper, for the ease of exposition, we restrict ourselves to the uniform weights, i.e., $\omega_i = \frac{n_i}{n}$.}.
\end{itemize}

In the above, assuming that each client has data of multiple classes, the loss function in (\ref{eq:loss}) can take various forms such as the contrastive loss \citep{hadsell2006constrastve, chopra2005contrastive}, triplet loss \citep{chechik2010large} and softmax cross-entropy. 
All such losses encourage two properties:
\begin{itemize}[noitemsep,topsep=0pt,leftmargin=0.2in]
    \item The embedding vector $g(\xv_j^i)$ and its positive class embedding $\wv_{y_j^i}$ are close. In other words, one wants large logits or scores for positives instance and label pairs.
    \item The embedding vector $g(\xv_j^i)$ and its negative class class embeddings $\wv_{c}$, $c \ne y_j^i$ are far away. In other words, one wants small logits or scores for negatives instance and label pairs.
\end{itemize}
For example, given a distance measure $\dist{\cdot}{\cdot}$, the contrastive loss is expressible as
\begin{align}
\label{eq:constrastive-loss}
&\ellcon\big(f(\xv), y\big) = \underbrace{\alpha\cdot \big(\dist{g_{\btheta}(\xv)}{\wv_y}\big)^2}_{\ellcon^{\rm pos}(f(\xv), y)} \; + \underbrace{\beta \cdot \sum_{c \neq y}\big(\max\big\{0, \nu - \dist{g_{\btheta}(\xv)}{\wv_c}\big\}\big)^2}_{\ellcon^{\rm neg}(f(\xv), y)},
\end{align}
where $\alpha, \beta \in \Real$ are some predefined constants. In \eqref{eq:constrastive-loss}, $\ellcon^{\rm pos}(\cdot)$ encourages high logit for the positive instance and label pairs. Similarly, $\ellcon^{\rm neg}(\cdot)$ aims to decrese the logit for the negative instance and label pairs.

\subsection{Federated Learning with only positive labels}
In this work, we consider the case where each client has access to only the data belonging to a single class. To simplify the notation, we assume that there are $m = C$ clients and the $i$-th client has access of the data of the $i$-th class. The algorithm and analysis also applies to the setting where multiple clients have the same class. 

The clients are not allowed to share their data with other clients, nor can they access the label embeddings associated with other clients. Formally, in each communication round, the $i$-th client has access to 
\begin{itemize}[noitemsep,topsep=0pt,leftmargin=0.2in]
    \item $n_i$ instance and label pairs with the same label $i$: $\SCal^{i} = \{(\xv^{i}_{1}, i),\ldots, (\xv^{i}_{n_i}, i)\} \subset \XCal \times [C]$
    \item Its own class embedding $\wv_i$.
    \item The current instance embedding model parameter $\btheta$.
\end{itemize}
Without access to the negative instance and label pairs, the loss function can only encourage the instances embedding and the positive class embedding to be close to each other. For example, with the contrastive loss in \eqref{eq:constrastive-loss}, in the absence of negative labels, one can only employ $\ellcon^{\rm pos}(\cdot)$ part of the loss function. Since $\ellcon^{\rm pos}(\cdot)$ is a monotonically decreasing function of the distance between the instance and the positive label, this approach would quickly lead to a trivial solution with small risk where all the users and the classes have an identical embedding. 
Regardless of the underlying loss function, training with only positive instance and label pairs will result in this degenerate solution. We propose an algorithm to address this problem in the next section. 

\section{Algorithm}
\label{sec:proposed}

To prevent all the class embeddings $\{\wv_i\}_{i=1}^C$ from collapsing into a single point in the optimization process, we propose Federated Averaging with Spreadout (FedAwS).

\subsection{Federated Averaging with Spreadout (FedAwS)}
In addition to Federated Averaging, the server performs an additional optimization step on the class embedding matrix $W \in \mathbb{R}^{C \times d}$ to ensure that different class embeddings are separated from each other by at least a margin of $\nu$. In particular, in each round of training, the server employs a geometric regularization, namely {\em spreadout regularizer}, which takes the following form. 
\begin{align}
\rsp(W) = \sum_{c \in [C]}\sum_{c' \neq c}\big(\max\big\{0, \nu - \dist{\wv_c}{\wv_{c'}}\big\}\big)^2.
\label{eq:sp}
\end{align}

A similar objective was first proposed as a regularizer to improve learning discriminative visual descriptors \citep{zhang2017learning} and then used in extreme-multiclass classification \citep{guo2019breaking}. There, it was shown that the spreadout regularization can improve the quality and stability of the learned models. In this work, we argue that the spreadout regularizer along with the positive part of the underlying loss function (e.g., $\ellcon^{\rm pos}(\cdot)$ in \eqref{eq:constrastive-loss}) constitutes a valid loss function that takes the similarity of the instance from both positive and negative labels into account (cf.~Section~\ref{sec:analysis}). This proves critical in realizing the meaningful training in the federated setting with only positive labels. 

\begin{algorithm}[t!]
\caption{Federated averaging with spreadout (FedAwS)}\label{alg:fedaws}
\begin{algorithmic}[1]
\STATE \textbf{Input.}~For $C$ clients and $C$ classes indexed by $[C]$, $n_i$ examples $\SCal_i$ at the $i$-th client.
\STATE Server initializes model parameters $\btheta^0, W^0$.
\FOR {$t = 0, 1,\ldots, T-1$}
\STATE The server communicates $\btheta^{t}, \wv^{t}_i$
to the $i$-th client.
\FOR {$i = 1, 2,\ldots, C$}
\STATE {The $i$-th client updates the model based on $\SCal_i$:}
\STATE $(\btheta^{t,i}, \wv_{i}^{t, i}) \gets (\btheta^{t}, \wv^{t}_i) - \eta\nabla_{(\btheta^{t}, \wv^{t}_{i})}\hat{\RCal}_{\rm pos }(\SCal^i),$\label{updatet}
\STATE $\text{where}~\hat{\RCal}_{\rm pos}(\SCal^i) = \frac{1}{n_i}\sum\limits_{j \in [n_i]}\ellcon^{\rm pos}(f(\xv), y).$
\STATE The $i$-th client sends $(\btheta^{t, i}, \wv^{t, i}_i)$ to the server.
\ENDFOR
\STATE Server updates the model parameters:
\STATE $ \btheta^{t+1} = \frac{1}{C}\sum\limits_{i\in[C]}\btheta^{t,i}$. \label{updatetheta}
\STATE $ \tilde{W}^{t+1} = [\wv^{t,i}_i, \dots, \wv^{t,C}_C]^T$.  \label{updatew1} 
\STATE  $ W^{t+1} \gets \tilde{W}^{t+1} - \lambda \eta \nabla_{\tilde{W}^{t+1}}\rsp(\tilde{W}^{t+1})$. \label{updatew2} 
\ENDFOR
\STATE {\bf Output:} $\btheta^{T}$ and $W^T$. 
\end{algorithmic}
\label{alg}
\end{algorithm}

The FedAwS algorithm which modifies the Federated Averaging using the spreadout regularizer is summarized in Algorithm \ref{alg}. Note that in Step~\ref{updatet}, the local objective at each client is define by the positive part $\ell^{\rm pos}(\cdot)$ of the the underling loss (cf.~\eqref{eq:constrastive-loss}). The algorithm differs from the conventional Federated Averaging in two ways. First, averaging of $W$ is replaced by updating the class embeddings received from each client (Step \ref{updatew1}). Second, an additional optimization step is performed on server to encourage the separation of the class embeddings (Step \ref{updatew2}). Here, we also introduce a learning rate multiplier $\lambda$ which controls the effect of the spreadout regularization term on the trained model.

\begin{remark}
In Algorithm \ref{alg}, we assumed all clients participate in each communication round for the ease of exposition. However, the algorithm easily extends to the practical setting, where only a subset of clients are involved in each round: Let ${\CCal}^t$ denote the set of clients participating the $t$-th round. Then, the server performs the updates in Step~\ref{updatetheta} and Step~\ref{updatew1} with the help of the information received from the clients indexed by ${\CCal}^t$. Note that the optimization in Step~\ref{updatet} and Step~\ref{updatew2} can employ multiple steps of SGD steps or based on other optimizers.
\end{remark}

\subsection{FedAwS with stochastic negative mining}
\label{sec:fedaws-snm}
There are two unique challenges that arise when we perform optimization w.r.t.~\eqref{eq:sp}. First, the best $\nu$ is problem dependent and therefore hard to choose. Second, when $C$ is large (also known as the extreme multiclass classification setting), even computing the spreadout regularizer becomes expensive. To this end we propose the following modification of \eqref{eq:sp}
\begin{align}
\rsptop (W) = \sum_{c \in {\CCal}^t}\sum\limits_{\substack{y \in \CCal', \\ y \ne c}} - \distsq{\wv_c}{\wv_y}\cdot \indicator{y \in \Near{k}{c}},
\label{eq:sampled_sp_topk}
\end{align}
where $\CCal'$ is a subset of classes, and $\Near{k}{c}$ denotes the set of $k$ classes that are closest to the class $c$ in the embedding space. The regularizer in \eqref{eq:sampled_sp_topk} can be viewed as an adaptive approximator of the spreadout regularizer in \eqref{eq:sp}, where, for each class $c$, we adaptively set $\nu$ to be the distance between $\wv_c$ and its $(k+1)$-th closest class embedding. Intuitively, we only need to make sure that, in the embedding space, each class is as far away as possible from its close classes.  

This approach of adaptively picking $\nu$ is motivated by the stochastic negative mining method first proposed in \citep{reddi2018stochastic}, where for each instance, they consider only the positive label and a small set of most confusing (`hard') negative labels to define the underlying loss function. On the contrary, we are picking the most confusing classes based on only the class embeddings. Furthermore, the methods is applied at the server as a regularizer as opposed to defining the underlying loss function for an individual instance. As we demonstrate in Section~\ref{sec:exp}, the stochastic negative mining is crucial to improve the quality of FedAwS. 

Before presenting these empirical results, we provide a theoretical justification for this in the following section.

\section{Analysis}
\label{sec:analysis}

To justify our FedAwS technique, we will:
\begin{enumerate}[label=(\roman*),itemsep=0pt,topsep=0pt]
    \item relate the classification error to the separation of the class embeddings
    \item introduce a particular \emph{cosine contrastive loss}, which we show to be \emph{consistent} for classification
    \item relate the FedAwS objective to empirical risk minimization using the cosine contrastive loss,
    despite the latter requiring both positive \emph{and} negative labels.
\end{enumerate}
Put together, this justifies why the FedAwS classifier 
can be close in performance to that of a consistent classifier,
despite only being trained with positive labels.

We first state a simple result arguing that
\emph{small} distance between the data embedding and the \emph{true} class embedding, 
and \emph{large} distance between the class embeddings, 
imply low classification error. 

\begin{proposition}
\label{thm:error}
Let the minimum distance between the class embeddings be $\rho := \inf_{i \ne j}\dist{\wv_i}{\wv_j}$, and the distance between the embeddings of an instance $\xv$ and its true class $y$ be $\epsilon = \mathbb{E}_{(\xv, y) \sim \pxy}\dist{g_{\btheta}(\xv)}{\wv_{y}}$. Then the probability of misclassification satisfies
\begin{align*}
P \big(\exists z \ne y~\text{s.t.}~\dist{g_{\btheta}(\xv)}{\wv_{y}} \geq \dist{g_{\btheta}(\xv)}{\wv_{z}}\big) \leq 2\epsilon/\rho.
\end{align*}
\end{proposition}

\begin{proof}
Note that, if there exists $z \neq y$ such that $\dist{g_{\btheta}(\xv)}{\wv_{y}} \geq \dist{g_{\btheta}(\xv)}{\wv_{z}}$, then 
\begin{align}
\label{eq:triangle-1}
\dist{g_{\btheta}(\xv)}{\wv_{y}} &\geq \frac{1}{2} \big(\dist{g_{\btheta}(\xv)}{\wv_{y}} + \dist{g_{\btheta}(\xv)}{\wv_{z}} \big)\\
&\overset{(i)}{\geq} \frac{\dist{\wv_y}{\wv_z}}{2} \overset{(ii)}{\geq} \frac{\rho}{2},
\end{align}
where $(i)$ and $(ii)$ follow from the triangle inequality and the definition of $\rho$, respectively. Next, by combing \eqref{eq:triangle-1} with Markov's inequality, we obtain that
\begin{align*}
P \big(\exists z \ne y~\text{s.t.}~\dist{g_{\btheta}(\xv)}{\wv_{y}} \geq \dist{g_{\btheta}(\xv)}{\wv_{z}}\big) &\leq  P \big(\dist{g_{\btheta}(\xv)}{\wv_{y}} \geq \frac{\rho}{2} \big) \\
& \leq \frac{2\mathbb{E}_{(\xv, y) \sim \pxy}\dist{g_{\btheta}(\xv)}{\wv_{y}}}{\rho} = \frac{2 \epsilon}{\rho}.
\end{align*}
\end{proof}

To relate the FedAwS objective to a contrastive loss, without loss of generality, we work with normalized embeddings; i.e., we assume that the rows of the matrix $W$ as well as the instance embeddings generated by $g_{\btheta}(\cdot)$ have unit Euclidean norm\footnote{The analysis in this section easily extends to unnormalized embeddings. However, the restriction to normalized embeddings slightly improves performance empirically.}.
We can then adopt the cosine distance:
\begin{align}
\label{eq:cosine-dist}
\cdist{\uv}{\uv'} = 1 - \uv^T\uv' \quad \forall~\uv,~\uv' \in \Real^d.
\end{align}
Specializing the contrastive loss in \eqref{eq:constrastive-loss} to the cosine distance measure gives us the \emph{cosine contrastive loss}.
\begin{definition}[Cosine contrastive loss] Given an instance and label pair $(\xv, y)$ and the scorer $f(\xv)$ in \eqref{eq:id-model-def}, the cosine contrastive loss takes the following form.
\begin{align}
\label{eq:contrastive-classification}
&\ellccon\big(f(\xv), y\big) =  \big(\cdist{g_{\btheta}(\xv)}{\wv_y}\big)^2 \; +  \sum_{c \neq y}\big(\max\big\{0, \nu - \cdist{g_{\btheta}(\xv)}{\wv_c}\big\}\big)^2.
\end{align}
Further, by using $s_c = g_{\btheta}^T(\xv)\wv_c$ to denote the logit for class $c$, the cosine contrastive loss can be expressed as
\begin{align}
\label{eq:contrastive-classification-1}
\ellccon\big(f(\xv), y\big) &= (1 - s_y)^2 \; + \sum_{c \neq y}{\big(\max\big\{0, \nu - 1 + s_c \big\}\big)^2}
\end{align}
\end{definition}
Note that, besides utilizing the cosine distance, we have used $\alpha = 1$ and $\beta = 1$ in \eqref{eq:constrastive-loss} to obtain \eqref{eq:contrastive-classification}. The following result states that cosine contrastive loss is a valid \emph{surrogate loss}~\citep{Bartlett:2006} for the misclassification error.

\begin{lemma}
Let $\nu \in (1, 2)$. The cosine contrastive loss in \eqref{eq:contrastive-classification-1} is a surrogate-loss of the misclassification error, i.e., 
\begin{align}
\ellccon\big(f(\xv), y\big) \geq 2(\nu - 1)\cdot\indicator{y \notin \Top{1}(f(\xv))},
\end{align}
where $\Top{1}(f(\xv))$ denotes the indices of the classes that $f(\cdot)$ assigns the highest score for the instance $\xv$.
\end{lemma}
\begin{proof}
If $y \in \Top{1}(f(\xv))$, then $\indicator{y \notin \Top{1}(f(\xv))} = 0$. Since $\ellccon\big(f(\xv), y\big) \geq 0$, in this case we have 
\begin{align}
\label{eq:surrogate-step1}
\ellccon\big(f(\xv), y\big) \geq 2(\nu - 1)\cdot\indicator{y \notin \Top{1}(f(\xv))}
\end{align}
in this case. 
Now, let's consider the case when $y \notin \Top{1}(f(\xv))$. For $a \in \Real$, let $\phi(a) =(1 - a)^2$ and $\tilde{\phi}(a) = (\max\{0, \nu - 1 - a \})^2$. With this notion, we have
\begin{align}
\label{eq:surrogate-step2}
\ellccon\big(f(\xv), y\big) &= \phi(s_y) \; + \sum_{c \neq y}\tilde{\phi}(-s_c) \nonumber \\
& {\geq} \phi(s_y) + \tilde{\phi}(-\max\limits_{c \neq y}s_c) \overset{(i)}{\geq} \tilde{\phi}(s_y) + \tilde{\phi}(-\max\limits_{c \neq y}s_c) \nonumber \\
& \overset{(ii)}{\geq} 2\cdot \tilde{\phi}\Big(\big({s_y -\max\limits_{c \neq y}s_c}\big)/{2}\Big) \overset{(iii)}{\geq} 2\cdot (\nu - 1) \nonumber\\ 
&= 2(\nu - 1)\cdot\indicator{y \notin \Top{1}(f(\xv))},
\end{align}
where $(i)$ follows as we have $\phi(a) \geq \tilde{\phi}(a), \forall~a$ and $(ii)$ utilizes the convexity of $\tilde{\phi}$. $(iii)$ follows as we have $\tilde{\phi}(a) > \nu - 1$, for $a < 0$, and 
$$
y \notin \Top{1}(f(\xv)) \quad \iff \quad s_y -\max\limits_{c \neq y}s_c < 0.
$$
The statement of the lemma follows from \eqref{eq:surrogate-step1} and \eqref{eq:surrogate-step2}.
\end{proof}

Having established that the cosine contrastive loss
is a valid surrogate,
one may follow similar analysis as in~\citet[Theorem 4]{reddi2018stochastic} to show the \emph{statistical consistency}~\citep{Zhang:2004} of minimizing this loss.

We now explicate a connection
between the classification-consistent cosine contrastive loss
and the objective underlying the FedAwS algorithm.
To do so, we assume that $n_1 = \cdots = n_{C} = \frac{n}{C}$,
and note that 
FedAwS 
effectively seeks to collaboratively minimize
\begin{align}
\label{eq:fedaws-objective}
&\RCal_{\rm sp}(f) = \sum\limits_{i \in [C]}\frac{n_i}{n}\cdot \hat{\RCal}_{\rm pos}(\SCal^{i}) + \lambda \cdot\rsp(W),
\end{align}
with 
$\rsp(W)$ the regulariser from~\eqref{eq:sp}.
Now we observe:

\begin{proposition}
\label{prop:fedaws_obj}
Suppose $\lambda = \frac{1}{C}$ and $n_1 = \cdots = n_{C} = \frac{n}{C}$. 
Then, FedAwS objective equals the empirical risk with respect to the loss function 
\begin{align}
\label{eq:lsp-def}
&\ellsp(f(\xv), y) =  (1 - s_y)^2 + \sum_{c \neq y}\big(\max\big\{0, \nu - 1 +  {\wv_y^T\wv_c}\big\}\big)^2,
\end{align}
i.e., $\RCal_{\rm sp}(f) = \frac{1}{n}\sum_{(\xv,y)\in \SCal} \ellsp(f(\xv), y)$.
\end{proposition}
\begin{proof}
Note that
\begin{align}
\label{eq:fedaws-obj}
\RCal_{\rm sp}(f) &= \sum\limits_{i \in [C]}\frac{n_i}{n}\cdot \hat{\RCal}_{\rm pos}(\SCal^{i}) + \lambda \cdot\rsp(W) \nonumber \\
&=\frac{1}{n} \sum_{(\xv,y)\in \SCal} \ellccon^{\rm pos}(f(\xv), y) + \lambda\cdot \rsp(W) \nonumber \\
&= \frac{1}{n} \sum_{(\xv,y)\in \SCal} \ellccon^{\rm pos}(f(\xv), y) + \lambda\sum_{y \in [C]}\sum_{c \neq y}\big(\max\big\{0, \nu - \cdist{\wv_y}{\wv_{c}}\big\}\big)^2 \nonumber \\
&\overset{(i)}{=} \frac{1}{n} \sum_{(\xv,y)\in \SCal} \Big( \ellccon^{\rm pos}(f(\xv), y) + {C\lambda}\sum_{c \neq y}\big(\max\big\{0, \nu - \cdist{\wv_y}{\wv_{c}}\big\}\big)^2 \Big) \nonumber \\
&\overset{(ii)}{=} \frac{1}{n} \sum_{(\xv,y)}\Big((1 - s_y)^2 + \sum_{c \neq y}\big(\max\big\{0, \nu - 1 +  {\wv_y^T\wv_c}\big\}\big)^2\Big) \nonumber \\
&= \frac{1}{n}\sum_{(\xv,y)\in \SCal} \ellsp(f(\xv), y),
\end{align}
where $(i)$ and $(ii)$ follows from the assumptions that $n_1 = \cdots = n_C$ and $\lambda =\frac{1}{C}$, respectively.
\end{proof}

Note that the contribution of the negative labels in the loss function $\ellsp$ is independent of the input embedding $g_{\btheta}(\xv)$. 

Recall from~\eqref{eq:constrastive-loss} that a contrastive loss has both a positive and negative component.
Proposition~\ref{prop:fedaws_obj} implies that $\ellsp^{\rm pos}(f(\xv), y) = \ellccon^{\rm pos}(f(\xv), y)$.
Next, we argue that $\ellsp^{\rm neg}(f(\xv), y)$ approximates $\ellccon^{\rm neg}(f(\xv), y)$.
This approximation becomes better as the input embedding $g_{\btheta}(\xv)$ gets closer to its class embedding $\wv_y$, as encouraged by
$\ellsp^{\rm pos}(f(\xv), y)$. 

\begin{theorem}
\label{lem:approx}
Let $\nu \in (1, 2)$. Then, the loss $\ellsp$ in \eqref{eq:lsp-def} satisfies
\begin{align}
&\ellccon(f(\xv), y) - (1 + 2\nu)\cdot\sum_{c \neq y}|\wv_c^T\rv_{\xv, y}| \leq \ellsp(f(\xv), y) \leq \ellccon(f(\xv), y) + (1 + 2\nu)\cdot\sum_{c \neq y}|\wv_c^T\rv_{\xv, y}|, \nonumber 
\end{align}
where $\rv_{\xv, y} =\wv_y - g_{\btheta}(\xv)$.
\end{theorem}

\begin{proof}
Note that $\rv_{\xv, y} =\wv_y - g_{\btheta}(\xv)$ denotes the mismatch between $\wv_y$ and $g_{\btheta}(\xv)$. Thus, 
$$\wv_y^T\wv_c = g_{\btheta}(\xv)^T\wv_c + \rv_{\xv, y}^T\wv_c = s_c + \rv_{\xv, y}^T\wv_c.$$
As a result $\ellsp$ in \eqref{eq:lsp-def} can be written as 
\begin{align}
\label{eq:approx-step1}
\ellps(f(\xv), y) &=  (1 - s_y)^2 + \sum_{c \neq y}\big(\max\big\{0, \nu - 1 +  s_c + \wv_c^T\rv_{\xv, y}\big\}\big)^2  \nonumber \\
&= (1 - s_y)^2 + \sum_{c \neq y}\big(\max\big\{0, \nu - 1 +  s_c \big\}\big)^2 + \sum_{c\neq y}{\Delta}_c\nonumber \\
& = \ellccon(f(\xv), y) + \sum_{c\neq y}{\Delta}_c,
\end{align}
where 
\begin{align}
\label{eq:delta-def}
&\Delta_c := \big(\max\big\{0, \nu - 1 +  s_c + \wv_c^T\rv_{\xv, y}\big\}\big)^2 -  \big(\max\big\{0, \nu - 1 +  s_c \big\}\big)^2.    
\end{align}
The result follows from \eqref{eq:approx-step1} and Claim~\ref{claim:sp-ccl} below.
\end{proof}

\begin{claim}
\label{claim:sp-ccl}
Given an instance and label pair $(\xv, y)$ and the scorer $f$, for $c \neq y$, let $\Delta_c$ be as defined in \eqref{eq:delta-def}. Then, 
\begin{align}
\vert\Delta_c\vert \leq 2(1 + 2\nu)\cdot\big|\wv_c^T\uv_{\xv, y}\big|.
\end{align}
\end{claim}

\begin{proof}
Let $a = \nu - 1 + s_c$ and $b = \wv_c^T\rv_{\xv,y}$. Thus, we want to show that
\begin{align*}
\big\vert\big(\max\big\{0, a + b\big\}\big)^2 - \big(\max\big\{0, a \big\}\big)^2\big\vert \leq (1 + 2\nu)\cdot\big|b\big|.
\end{align*}
Let us consider four possible cases. 
\begin{itemize}
    \item \textbf{Case 1 ($a + b < 0$ and $a < 0$).}~In this case, we have
$$
\big\vert\big(\max\big\{0, a + b\big\}\big)^2 - \big(\max\big\{0, a \big\}\big)^2\big\vert = 0.
$$
\item \textbf{Case 2 ($a + b > 0$ and $a > 0$).}~Note that
\begin{align*}
&\big\vert\big(\max\big\{0, a + b\big\}\big)^2 - \big(\max\big\{0, a \big\}\big)^2\big\vert = \vert (a + b)^2 - a^2 \vert = \vert|b(b + 2a)| \leq (1 + 2\nu)\cdot|b|,
\end{align*}
where the last inequality follows from the fact that $a = \nu - 1 + s_c \leq \nu$, since $s_c \leq 1$.
\item \textbf{Case 3 ($a + b > 0$ and $a < 0$).}~In this case,
\begin{align*}
&\big\vert\big(\max\big\{0, a + b\big\}\big)^2 - \big(\max\big\{0, a \big\}\big)^2\big\vert = \big\vert\max\big\{0, a + b\big\}\big)^2\big\vert \leq |b^2| \leq |b|,      
\end{align*}
where the last equality follows as $|b| = \vert\wv_c^T\rv_{\xv, y}\vert \leq 1$.
\item \textbf{Case 4 ($a + b < 0$ and $a > 0$).}~Note that
\begin{align*}
&\big\vert\big(\max\big\{0, a + b\big\}\big)^2 - \big(\max\big\{0, a \big\}\big)^2\big\vert = \big\vert\max\big\{0, a\big\}\big)^2\big\vert \leq |a|^2 \overset{(i)}{\leq} |b|^2 \leq |b|,
\end{align*}
where $(i)$ follows as by combining $a > 0$ and $a+b < 0$ we obtain the order $b < -a < 0 < a$.
\end{itemize}
Now, by combining all the four case above and using the fact that $\nu \in (1, 2)$, we obtain the desired the result.
\end{proof}

As a final remark, our analysis above assumed that the cosine contrastive loss~\eqref{eq:contrastive-classification} uses \emph{all} labels $c \neq y$ as ``negatives'' for the given label $y$. However, using similar ideas as in~\citep{reddi2018stochastic}, we may easily extend our analysis to the case where the loss uses the \emph{$k$ hardest labels} as negatives (cf.~\eqref{eq:sampled_sp_topk}).


\begin{table*}
	\centering
        \small
	\begin{tabular}{ l | l | l | l | l | l   }
		\hline
		Dataset & Model & Baseline-1  & Baseline-2 & FedAwS &Softmax (Oracle) \\
		\hline
		\cifars &\resnet-8 &   10.7   & 83.3    & 86.3 & 88.4 \\		
		\hline
		\cifars &\resnet-32 & 9.8     & 92.1    & 92.4 & 92.4 \\
		\hline    
		\cifarl  & \resnet-32 &1.0    & 65.1   & 67.9 & 68.0 \\
		\hline
		\cifarl  & \resnet-56 &  1.1   & 67.5   & 69.6 & 70.0 \\
		\hline		
	\end{tabular}
	\caption{Precision@1 (\%) on \cifars~and \cifarl.}
	\label{tb:cifar}
\end{table*}


\section{Experiments}
\label{sec:exp}

We empirically evaluate the proposed FedAwS method on benchmark image classification and extreme multi-class classification datasets. 
In all experiments, both the class embedding $\wv_c$'s and instance embedding $g_{\btheta}(\xv)$ are $\ell_2$ normalized, as we found this slightly improves model quality. 

For FedAwS, we use the squared hinge loss with cosine distance 
to define $\hat{\RCal}_{\rm pos}(\SCal^i)$ at the clients (cf.~Algorithm~\ref{alg:fedaws}):
\begin{align}
    \ell^{\rm pos}(f(\xv), y) =  \max\big(\big\{0, 0.9 - g_{\btheta}(\xv)^T \wv_{y} \big\}\big)^2.
\end{align}
This encourages all positive instance and label pairs $(\xv, y)$ to have dot product larger than 0.9 in the embedding space.

We compare the following methods in our  experiments.
\begin{itemize}[noitemsep,topsep=0pt,leftmargin=0.2in]
\item {\bf Baseline-1}: Training with only positive squared hinge loss. As expected, we observe very low precision values because the model quickly collapses to a trivial solution.
\item {\bf Baseline-2}: Training with only positive squared hinge loss with the class embeddings fixed. This is a simple way of preventing the class embeddings from collapsing into a single point. 
\item {\bf FedAwS}: Our method with stochastic negative mining (cf.~Section~\ref{sec:fedaws-snm}).
\item {\bf Softmax}: An oracle method of regular training with the softmax cross-entropy loss function that has access to both positive and negative labels.
\end{itemize}

\subsection{Experiments on CIFAR}

We first present results on the CIFAR-10 and CIFAR-100 datasets.
We trained ResNets (\resnets)~\citep{he2016deep, he2016identity} with different number of layers as the underlying model. Specifically, we train \resnet-8 and \resnet-32 for \cifars; and train \resnet-32 and \resnet-56 for \cifarl~with the larger number of classes.

From Table~\ref{tb:cifar},
we see that
on both \cifars~and \cifarl, FedAwS almost matches or comes very close to the performance of the oracle method which has access to all labels. 
The first baseline method, training with only positive squared hinge loss does not lead to any meaningful precision values. In this case, as discussed above the model collapses into a degenerate solution.

Interestingly, the naive way of preventing the embeddings from collapsing by fixing the class embeddings as their random initialization gives a much better result. In fact, on \cifars~with \resnet-32, Baseline-2  performs almost identically to the oracle and FedAwS. The reason behind this good performance is that with a smaller number of classes, at a random initialization in a high-dimensional space ($64$ in this case), the class embeddings are already well spread-out as they are almost orthogonal to each other. In addition, the 10 classes of \cifars~are not related to each other. This makes the 10 nearly-orthogonal vectors ideal to be used as-is for class embeddings.


\begin{table*}
	\centering
        \small
	\begin{tabular}{ l | l | l | l | l | l | l  }
		\hline
		Dataset & \#Features & \#Labels & \#TrainPoints &\#TestPoints & Avg. \#I/L & Avg. \#L/I  \\
		\hline
		\amcat  & 203,882     & 13,330    & 1,186,239 & 306,782 & 448.57 & 5.04\\
		\hline    
		\lshtc  & 1,617,899    & 325,056   & 1,778,351 & 587,084 & 17.46 & 3.19\\
		\hline
		\amsmall& 135,909     & 670,091   & 490,449  & 153,025 & 3.99  & 5.45\\
		\hline
	\end{tabular}
	\caption{Summary of the datasets used in the paper. \#I/L is the number of instances per label, and \#L/I is the number of labels per instance.}\label{tb:datasets}
\end{table*}

\begin{table*}
\centering
\begin{tabular}{ccccc|cc}
\multicolumn{2}{c}{}
& \multicolumn{3}{c}{\bf Federated Learning with Only Positives}
& \multicolumn{2}{c}{\bf Oracle} \\
\hline
& &Baseline-1 & Baseline-2 & FedAwS & Softmax & SLEEC \\
\hline
& P@1        & 3.4 & 64.1 & 92.1 & 92.1 & 90.5 \\
\amcat & P@3 & 3.2 & 46.8 & 70.8 & 77.9 & 76.3  \\
& P@5        & 3.1 & 32.6 & 58.7 & 62.3 & 61.5  \\
\hline
& P@1          & 0.0 & 4.3 & 33.1 & 35.2 & 35.1  \\
\amsmall & P@3 & 0.0 & 2.8 & 29.6 & 31.6 & 31.3 \\
& P@5          & 0.0 & 2.2 & 27.4 & 29.5 & 28.6\\
\hline
& P@1        & 7.6 & 7.9 & 37.2 & 54.1 & 54.8  \\
\lshtc & P@3 & 4.5 & 3.4 & 22.6 & 38.8 & 33.4 \\
& P@5        & 2.8 & 2.6 & 16.2  & 29.9 & 23.9  \\
\hline
\end{tabular}
\caption{P@1,3,5 (\%) of different methods on \amcat, \amsmall~ and \lshtc.}
\label{tb:pre-table}
\end{table*}

\subsection{Experiments on extreme-multiclass classification}

\textbf{Datasets.}
We test the proposed approach on standard extreme multilabel classification datasets~\citep{V18}. These datasets have a large number of classes, and therefore are a good representatives of the applications of federated learning with only positive labels. Similar to \citep{reddi2018stochastic}, because these datasets are multi-label, we uniformly sample positive labels to obtain datasets corresponding to multi-class classification problems. The datasets and their statistics are summarized in Table \ref{tb:datasets}.

\textbf{Model architecture.}
We use
a simple embedding-based classification model
wherein
an instance $\xv \in \mathbb{R}^{d'}$, a high-dimensional sparse vector, is first embedded into $\Real^{512}$ using a linear embedding lookup followed by averaging. The vector is then passed through a three-layer neural network with layer sizes $1024$, $1024$ and $512$, respectively. The first two layers in the network apply a ReLU activation function. The output of the network is then normalized to obtain instance embeddings with unit $\ell_2$-norm. Each class is represented as a $512$-dimensional normalized vector.

\textbf{Training setup.} SGD with a large learning rate is used to optimize the embedding layers, and Adagrad is used to update other model parameters. In each round, we randomly select 4K clients associated with 4K labels.

In addition to the methods used in the CIFAR experiments, we also compare the FedAwS with SLEEC \cite{hadsell2006constrastve}.
This is an oracle method of regular training with access to both positive and  negative labels.

\textbf{Results}.
We report precision@$k$ for $k \in \{ 1, 3, 4 \}$ 
in Table \ref{tb:pre-table}. 
On all the datasets, FedAwS largely outperforms the two baseline methods of training with only positive labels. On both \amcat~and \amsmall, it matches or comes very close to the performance of Softmax and SLEEC. Baseline-2 gives reasonable (although quite sub-optimal) performance on \amcat; but does not work on \amsmall~and \lshtc~which have larger number of classes. Thus, randomly initialized class embeddings are not ideal in the situation of many classes, and it is crucial to train the class embeddings with the rest of the model.

\begin{table*}[t!]
\centering
\begin{tabular}{ccc|cccc|ccc}
\hline
& Baseline-1 & Baseline-2  &k = 10 & k = 100 & k = 500 & k = all & $\lambda$ = 1 & $\lambda$ = 10  & $\lambda = 100 $\\
\hline
P@1    & 3.4 & 64.1  & 26.3 & 92.1 & 86.9 & 87.7 & 73.2 & 92.1 & 92.2  \\
P@3    & 3.2 & 46.8  & 21.5 & 70.8 & 66.1 & 69.7 & 50.2 & 70.8 & 71.7 \\
P@5    & 3.1 & 32.6  & 18.2 & 58.7 & 49.3 & 52.2 & 40.4 & 58.7 & 57.9 \\
\hline
\end{tabular}
\caption{P@1,3,5 (\%) of different meta parameters on \amcat.}
\label{tb:meta}
\end{table*}

\textbf{Meta parameters.} There are two meta parameters in the proposed method: the learning rate multiplier of the spreadout loss $\lambda$ (cf.~Algorithm~\ref{alg:fedaws}), and the number top confusing labels considered in each round $k$ (cf.~\eqref{eq:sampled_sp_topk}). 
To make a fair comparison with other methods which do not have these meta parameters, in all of our other experiments in Table \ref{tb:pre-table}, we simply use $k = 10$ and $\lambda = 10$. 

We perform an analysis of these two parameters in Table \ref{tb:meta} on the \amcat~dataset. A very large $k$  leads to worse performance, verifying the benefit and requirement of stochastic negative mining. The reason for the bad performance for a small $k$ is that most of the picked labels are in fact positives in this setting (due to the inherent multi-label nature of the dataset), and over spreading the positive classes is not desirable. Regarding $\lambda$, a relatively large value such as 10 or 100 is necessary to ensure that the class embeddings are sufficiently spreadout.


\section{Conclusion}
We studied a novel learning setting, federated learning with only positive labels, and proposed an algorithm that can learn a high-quality classification model without requiring negative instance and label pairs. The idea is to impose a geometric regularization on the server side to make all class embeddings spreadout. We justified the proposed method both theoretically and empirically. 
For future directions, one can extend the id based class embeddings to the settings where the class embeddings are generated from class-level features. In addition, we notice that negative sampling techniques are crucial to make conventional extreme multiclass classification work. The proposed method is of independent interest in this setting because it replaces negative sampling all together by imposing a strong geometric regularization.

\bibliography{federated}

\begin{thebibliography}{27}
\providecommand{\natexlab}[1]{#1}
\providecommand{\url}[1]{\texttt{#1}}
\expandafter\ifx\csname urlstyle\endcsname\relax
  \providecommand{\doi}[1]{doi: #1}\else
  \providecommand{\doi}{doi: \begingroup \urlstyle{rm}\Url}\fi

\bibitem[Abadi et~al.(2016)Abadi, Chu, Goodfellow, McMahan, Mironov, Talwar,
  and Zhang]{abadi2016deep}
Abadi, M., Chu, A., Goodfellow, I., McMahan, H.~B., Mironov, I., Talwar, K.,
  and Zhang, L.
\newblock Deep learning with differential privacy.
\newblock In \emph{Proceedings of the 2016 ACM SIGSAC Conference on Computer
  and Communications Security}, pp.\  308--318, 2016.

\bibitem[Agarwal et~al.(2018)Agarwal, Suresh, Yu, Kumar, and
  McMahan]{agarwal2018cpsgd}
Agarwal, N., Suresh, A.~T., Yu, F. X.~X., Kumar, S., and McMahan, B.
\newblock {cpSGD}: Communication-efficient and differentially-private
  distributed sgd.
\newblock In \emph{Advances in Neural Information Processing Systems}, pp.\
  7564--7575, 2018.

\bibitem[Augenstein et~al.(2019)Augenstein, McMahan, Ramage, Ramaswamy,
  Kairouz, Chen, Mathews, et~al.]{augenstein2019generative}
Augenstein, S., McMahan, H.~B., Ramage, D., Ramaswamy, S., Kairouz, P., Chen,
  M., Mathews, R., et~al.
\newblock Generative models for effective ml on private, decentralized
  datasets.
\newblock \emph{arXiv preprint arXiv:1911.06679}, 2019.

\bibitem[Bartlett et~al.(2006)Bartlett, Jordan, and McAuliffe]{Bartlett:2006}
Bartlett, P.~L., Jordan, M.~I., and McAuliffe, J.~D.
\newblock Convexity, classification, and risk bounds.
\newblock \emph{Journal of the American Statistical Association}, 101\penalty0
  (473):\penalty0 138--156, 2006.

\bibitem[Bonawitz et~al.(2016)Bonawitz, Ivanov, Kreuter, Marcedone, McMahan,
  Patel, Ramage, Segal, and Seth]{bonawitz2016practical}
Bonawitz, K., Ivanov, V., Kreuter, B., Marcedone, A., McMahan, H.~B., Patel,
  S., Ramage, D., Segal, A., and Seth, K.
\newblock Practical secure aggregation for federated learning on user-held
  data.
\newblock \emph{arXiv preprint arXiv:1611.04482}, 2016.

\bibitem[Chechik et~al.(2010)Chechik, Sharma, Shalit, and
  Bengio]{chechik2010large}
Chechik, G., Sharma, V., Shalit, U., and Bengio, S.
\newblock Large scale online learning of image similarity through ranking.
\newblock \emph{Journal of Machine Learning Research}, 11\penalty0
  (Mar):\penalty0 1109--1135, 2010.

\bibitem[Chopra et~al.(2005)Chopra, Hadsell, and LeCun]{chopra2005contrastive}
Chopra, S., Hadsell, R., and LeCun, Y.
\newblock Learning a similarity metric discriminatively, with application to
  face verification.
\newblock In \emph{Computer Vision and Pattern Recognition}, pp.\  539--546,
  2005.

\bibitem[Dietterich \& Bakiri(1991)Dietterich and Bakiri]{dietterich1991error}
Dietterich, T.~G. and Bakiri, G.
\newblock Error-correcting output codes: A general method for improving
  multiclass inductive learning programs.
\newblock In \emph{AAAI}, pp.\  572--577, 1991.

\bibitem[Elkan \& Noto(2008)Elkan and Noto]{elkan08}
Elkan, C. and Noto, K.
\newblock Learning classifiers from only positive and unlabeled data.
\newblock In \emph{ACM SIGKDD International Conference on Knowledge Discovery
  and Data Mining}, pp.\  213–220, 2008.

\bibitem[Guo et~al.(2019)Guo, Mousavi, Wu, Holtmann-Rice, Kale, Reddi, and
  Kumar]{guo2019breaking}
Guo, C., Mousavi, A., Wu, X., Holtmann-Rice, D.~N., Kale, S., Reddi, S., and
  Kumar, S.
\newblock Breaking the glass ceiling for embedding-based classifiers for large
  output spaces.
\newblock In \emph{Advances in Neural Information Processing Systems}, pp.\
  4944--4954, 2019.

\bibitem[Hadsell et~al.(2006)Hadsell, Chopra, and
  LeCun]{hadsell2006constrastve}
Hadsell, R., Chopra, S., and LeCun, Y.
\newblock Dimensionality reduction by learning an invariant mapping.
\newblock In \emph{Computer Vision and Pattern Recognition}, volume~2, pp.\
  1735--1742, 2006.

\bibitem[He et~al.(2016{\natexlab{a}})He, Zhang, Ren, and Sun]{he2016deep}
He, K., Zhang, X., Ren, S., and Sun, J.
\newblock Deep residual learning for image recognition.
\newblock In \emph{Proceedings of the IEEE conference on computer vision and
  pattern recognition}, pp.\  770--778, 2016{\natexlab{a}}.

\bibitem[He et~al.(2016{\natexlab{b}})He, Zhang, Ren, and Sun]{he2016identity}
He, K., Zhang, X., Ren, S., and Sun, J.
\newblock Identity mappings in deep residual networks.
\newblock In \emph{European conference on computer vision}, pp.\  630--645.
  Springer, 2016{\natexlab{b}}.

\bibitem[Hsieh et~al.(2015)Hsieh, Natarajan, and Dhillon]{hsiehb15}
Hsieh, C.-J., Natarajan, N., and Dhillon, I.
\newblock Pu learning for matrix completion.
\newblock In \emph{Proceedings of the 32nd International Conference on Machine
  Learning}, volume~37, pp.\  2445--2453. PMLR, 07--09 Jul 2015.

\bibitem[Li et~al.(2019)Li, Sahu, Talwalkar, and Smith]{li2019federated}
Li, T., Sahu, A.~K., Talwalkar, A., and Smith, V.
\newblock Federated learning: Challenges, methods, and future directions.
\newblock \emph{arXiv preprint arXiv:1908.07873}, 2019.

\bibitem[Liu et~al.(2002)Liu, Lee, Yu, and Li]{liu2002partially}
Liu, B., Lee, W.~S., Yu, P.~S., and Li, X.
\newblock Partially supervised classification of text documents.
\newblock In \emph{International Conference on Machine Learning}, volume~2,
  pp.\  387--394, 2002.

\bibitem[Manevitz \& Yousef(2001)Manevitz and Yousef]{manevitz2001one}
Manevitz, L.~M. and Yousef, M.
\newblock One-class svms for document classification.
\newblock \emph{Journal of machine Learning research}, 2\penalty0
  (Dec):\penalty0 139--154, 2001.

\bibitem[McMahan et~al.(2017)McMahan, Moore, Ramage, Hampson, and
  y~Arcas]{mcmahan2017communication}
McMahan, B., Moore, E., Ramage, D., Hampson, S., and y~Arcas, B.~A.
\newblock Communication-efficient learning of deep networks from decentralized
  data.
\newblock In \emph{Artificial Intelligence and Statistics}, pp.\  1273--1282,
  2017.

\bibitem[Mohri et~al.(2019)Mohri, Sivek, and Suresh]{mohri2019agnostic}
Mohri, M., Sivek, G., and Suresh, A.~T.
\newblock Agnostic federated learning.
\newblock In \emph{International Conference on Machine Learning}, pp.\
  4615--4625, 2019.

\bibitem[Moya \& Hush(1996)Moya and Hush]{moya1996network}
Moya, M.~M. and Hush, D.~R.
\newblock Network constraints and multi-objective optimization for one-class
  classification.
\newblock \emph{Neural Networks}, 9\penalty0 (3):\penalty0 463--474, 1996.

\bibitem[Plessis et~al.(2015)Plessis, Niu, and Sugiyama]{plessis15}
Plessis, M.~D., Niu, G., and Sugiyama, M.
\newblock Convex formulation for learning from positive and unlabeled data.
\newblock In \emph{Proceedings of the 32nd International Conference on Machine
  Learning}, volume~37, pp.\  1386--1394, Lille, France, 07--09 Jul 2015. PMLR.

\bibitem[Pujol et~al.(2006)Pujol, Radeva, and Vitria]{pujol2006discriminant}
Pujol, O., Radeva, P., and Vitria, J.
\newblock Discriminant ecoc: A heuristic method for application dependent
  design of error correcting output codes.
\newblock \emph{IEEE Transactions on Pattern Analysis and Machine
  Intelligence}, 28\penalty0 (6):\penalty0 1007--1012, 2006.

\bibitem[Reddi et~al.(2019)Reddi, Kale, Yu, Holtmann-Rice, Chen, and
  Kumar]{reddi2018stochastic}
Reddi, S.~J., Kale, S., Yu, F., Holtmann-Rice, D., Chen, J., and Kumar, S.
\newblock Stochastic negative mining for learning with large output spaces.
\newblock \emph{Artificial Intelligence and Statistics}, 2019.

\bibitem[Varma(2018)]{V18}
Varma, M.
\newblock Extreme classification repository.
\newblock Website, 8 2018.
\newblock \url{http://manikvarma.org/downloads/XC/XMLRepository.html}.

\bibitem[Yu et~al.(2013)Yu, Cao, Feris, Smith, and Chang]{yu2013designing}
Yu, F.~X., Cao, L., Feris, R.~S., Smith, J.~R., and Chang, S.-F.
\newblock Designing category-level attributes for discriminative visual
  recognition.
\newblock In \emph{Computer Vision and Pattern Recognition}, pp.\  771--778,
  2013.

\bibitem[Zhang(2004)]{Zhang:2004}
Zhang, T.
\newblock Statistical behavior and consistency of classification methods based
  on convex risk minimization.
\newblock \emph{Ann. Statist.}, 32\penalty0 (1):\penalty0 56--85, 02 2004.

\bibitem[Zhang et~al.(2017)Zhang, Yu, Kumar, and Chang]{zhang2017learning}
Zhang, X., Yu, F.~X., Kumar, S., and Chang, S.-F.
\newblock Learning spread-out local feature descriptors.
\newblock In \emph{International Conference on Computer Vision}, pp.\
  4595--4603, 2017.

\end{thebibliography}
\bibliographystyle{icml2020}

\end{document}